\documentclass{article} 

\usepackage{amsmath,amsfonts,bm}

\def\eqref#1{equation~\ref{#1}}

\def\1{\bm{1}}

\DeclareMathAlphabet{\mathsfit}{\encodingdefault}{\sfdefault}{m}{sl}
\SetMathAlphabet{\mathsfit}{bold}{\encodingdefault}{\sfdefault}{bx}{n}

\DeclareMathOperator*{\argmax}{arg\,max}

\usepackage{hyperref}
\usepackage{url}
\usepackage{graphicx}
\usepackage{booktabs}
\usepackage{caption}
\usepackage{subfigure}
\usepackage{wrapfig}
\usepackage{algorithm}
\usepackage{algpseudocode}

\usepackage{natbib}
\usepackage[margin=0.9in]{geometry}

\usepackage{amsthm}
\newtheorem{theorem}{Theorem}[section]

\newtheorem{lemma}[theorem]{Lemma}

\usepackage{xcolor}
\hypersetup{
    colorlinks,
    linkcolor={red!50!black},
    citecolor={blue!50!black},
    urlcolor={blue!80!black}
}

\usepackage{enumitem}
\setlist[itemize]{noitemsep, nolistsep, leftmargin=*}
\setlist[enumerate]{noitemsep, nolistsep, leftmargin=*}

\title{Large Language Model Routing with Benchmark Datasets}

\date{}

\author{Tal Shnitzer\thanks{Equal contribution} \thanks{Broad Institute, work done while at CSAIL, MIT}, Anthony Ou\footnotemark[1] \thanks{CSAIL, MIT}, Mírian Silva\thanks{MIT-IBM Watson AI Lab}, Kate Soule\footnotemark[4], Yuekai Sun\thanks{University of Michigan}, \\ Justin Solomon\footnotemark[3], Neil Thompson\footnotemark[3], Mikhail Yurochkin\footnotemark[4]}

\begin{document}

\maketitle

\begin{abstract}
There is a rapidly growing number of open-source Large Language Models (LLMs) and benchmark datasets to compare them. 
While some models dominate these benchmarks, no single model typically achieves the best accuracy in all tasks and use cases.
In this work, we address the challenge of selecting the best LLM out of a collection of models for new tasks.
We propose a new formulation for the problem, in which benchmark datasets are repurposed to learn a ``router'' model for this LLM selection, and we show that this problem can be reduced to a collection of binary classification tasks. We demonstrate the utility and limitations of learning model routers from various benchmark datasets, where we consistently improve performance upon using any single model for all tasks.
\end{abstract}

\section{Introduction}

\begin{wrapfigure}[27]{r}{0.43\linewidth}
\vspace{-0.4in}
    \centering
    \includegraphics[width=\linewidth]
    {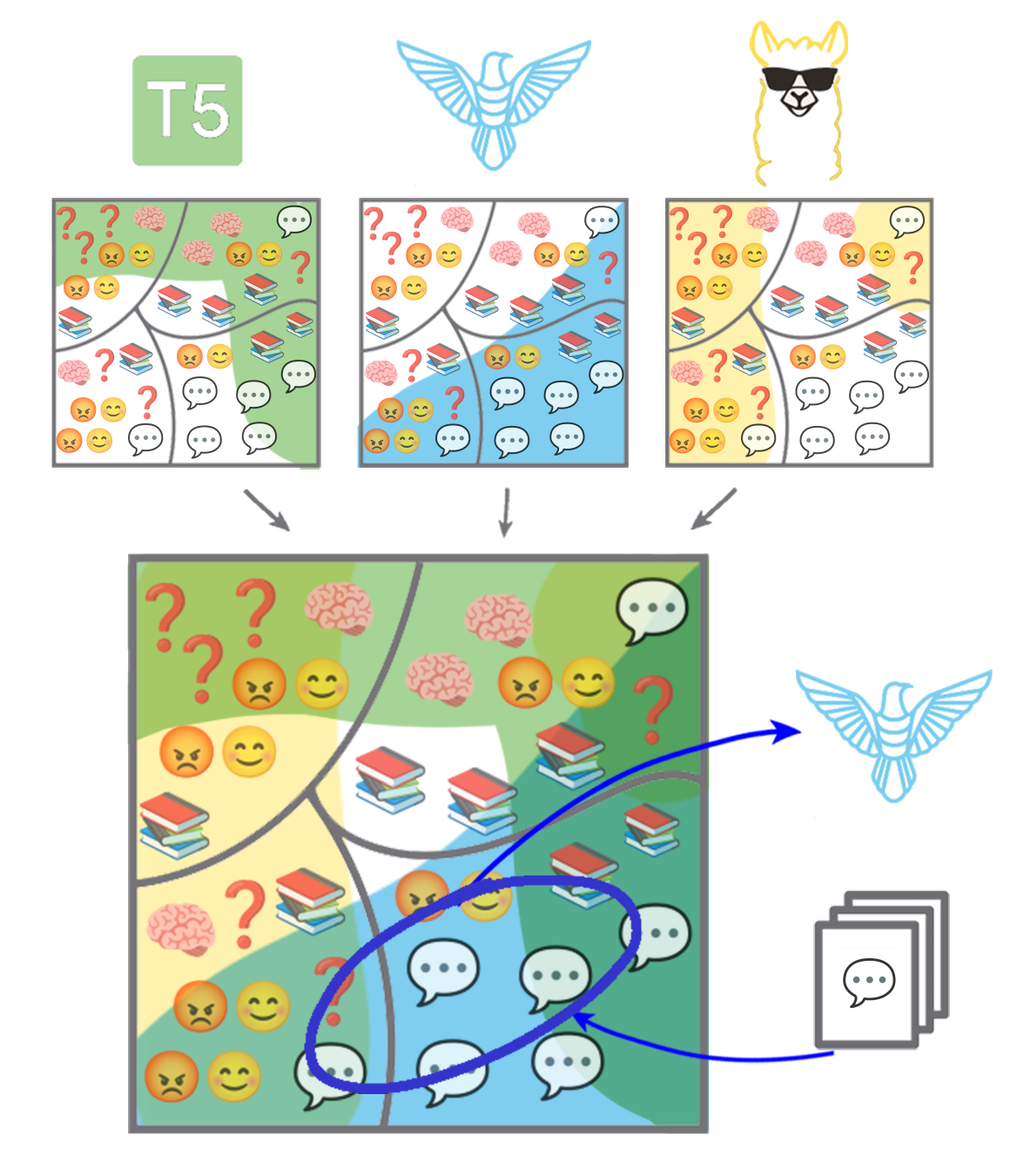}
    \vspace{-0.25in}
    \caption{We learn the strengths of candidate LLMs (marked with corresponding colors) on various tasks (emojis: QA, reasoning, summarization, etc.)  and domains (4 sections within each box: finance, legal, general knowledge, etc.) from benchmark datasets. We accomplish this by training a binary classifier per LLM (upper part of the figure). For a new task, we score each LLM with these binary classifiers and recommend an LLM for the user (lower part).}
    \label{fig:diagram}
\end{wrapfigure}

Large Language Models (LLMs) have demonstrated groundbreaking abilities to solve diverse tasks across a variety of NLP domains \citep{devlin2018bert,brown2020language}. Today, researchers in both academia and industry are releasing new LLMs \emph{daily}.\footnote{Hugging Face currently hosts 22,482 models for text generation.} These models perform tasks ranging from text classification to question-answering, summarization, and dialogue.

The popularity and influx of open-source LLMs and the diversity of their potential use cases made it crucial to develop comprehensive benchmarks, i.e., collections of datasets representing different tasks and domains to compare LLMs. For example, HELM \citep{liang2022holistic} consists of 42 scenarios covering a variety of uses, MMLU \citep{hendrycks2020measuring} is a multiple-choice question answering benchmark with 57 tasks organized by topics, Open LLM Leaderboard \citep{open-llm-leaderboard} combines MMLU with other question-answering datasets, and LM Evaluation Harness \citep{eval-harness} supports over 200 tasks. While there always will be an LLM that is the best \emph{on average} across benchmarks, there is unlikely to ever be a model that is strictly the best \emph{on each} of the hundreds of datasets comprising various benchmarks. Meanwhile, a practitioner typically wants to know what is the best model for their specific use case and is less concerned about average performance on a plethora of other datasets.

In this paper, we study the problem of identifying the best LLM for a new task. To learn about the strengths and weaknesses of candidate LLMs we use benchmark datasets that give insights into the performance of LLMs across tasks and domains. For example, suppose the new task is answering math questions. In that case, it is more intuitive to consider models that do well on other STEM question-answering datasets and discount performance on, e.g., sociology or toxicity detection. We make this idea precise by casting the learning of model strengths as a binary supervised learning task, where the features are input embeddings of samples across tasks and the labels are whether the model ``did well'' on the corresponding inputs, e.g., generated correct class label, answered a question correctly, or followed input instructions sufficiently well. See Figure \ref{fig:diagram} for an illustration. Such information is collected during benchmark evaluations and can be reused for training model routers without having to run expensive LLM inference again. The resulting router is also efficient at test time, as it only requires calling the chosen LLM.

Our contributions are summarized below:

\begin{itemize}
    \item We formalize the problem of learning the strengths and weaknesses of LLMs for downstream \emph{routing}, i.e., selecting the best model, as a collection of binary classification problems. The goal of each classification problem is to predict whether a given LLM will be ``correct'' on an input.
    \item We propose three scores for selecting LLMs for a new task using these correctness predictors. Our third score is designed to account for mistakes a correctness predictor can make on the (out-of-distribution) data from a new task that is likely to be different from the datasets in benchmarks used for training the correctness predictors. We establish connections to meta-learning to obtain theoretical insights into the efficacy of these scores.
    \item We verify the efficiency of our model routing scores empirically on 29 datasets from HELM \citep{liang2022holistic} representing scenarios like question answering, text classification, knowledge, and reasoning, and MixInstruct \citep{jiang2023llm}, a collection of datasets for evaluating instruction following capabilities of LLMs.
    \item We discuss and empirically investigate questions concerning the efficacy and utility of learning LLM routers from benchmarks: generalization of correctness predictors to new tasks, the importance of a larger pool of benchmarks, and the potential of routing smaller LLMs to reduce costs.
\end{itemize}

\section{Related work}

\paragraph{Benchmarking} Comparing models or algorithms across various tasks is a standard practice in ML and AI literature. Prior to Foundation Models \citep{bommasani2021opportunities}, it was typical to apply \emph{the same learning algorithm} to train a model on each of the datasets and compare the performance against other learning algorithms. The UCI Machine Learning Repository \citep{uci} is one prominent example of such a collection of datasets often used to compare learning algorithms. With the emergence of Foundation Models, i.e., models with billions of parameters trained on massive datasets using large compute clusters, the paradigm changed to evaluating \emph{the same model} (or a few-shot tuned version of it) on a variety of tasks \citep{bojar2014findings,goyal2019scaling,li2022elevater}. In the context of Large Language Models, many benchmarks \citep{wang2018glue,wang2019superglue,hendrycks2020measuring,eval-harness,srivastava2022beyond,liang2022holistic,open-llm-leaderboard,jiang2023llm} were proposed to help determine the most capable LLM. Benchmarks typically average the performance of models across tasks and provide a final ranking, discarding the rest of the information. In this work, we use the byproducts of benchmark evaluations, i.e., the per-sample performance of various LLMs across tasks, to learn about their individual strengths and identify the best LLM for a new task.

\paragraph{Model selection} Selecting the best model, or model selection, is a classical topic in statistics and ML \citep{bishop2006pattern,hastie2009elements,raschka2018model}. However, the typical problem setting is quite different: classical methods like cross-validation aim to estimate the population error of a model trained on samples from the population distribution. In other words, the goal is to find the best model for in-distribution test data, i.e., data sampled from the same distribution as the train data. The notion of ``train'' data is quite elusive for LLMs, as they are usually trained on massive datasets with trillions of tokens with a simple task of next token prediction \citep{radford2019language,brown2020language}. However, the tasks we evaluate them on are often more structured, e.g., classification and question-answering, and are specific to domains that may or may not be sufficiently represented in the train data. In addition, techniques like $k$-fold cross-validation require training the model multiple times, which is infeasible for LLMs.

\paragraph{Out-of-distribution model selection} Recognizing the limitations of the model selection methods for in-distribution test data \citep{gulrajani2020search,koh2021wilds}, recent work has proposed a variety of methods to select models when deployed on data that may differ from the train data. These methods rely on ideas such as bootstrapping \citep{xu2022estimation}, reweighing \citep{chen2021mandoline,maity2022understanding}, agreement of models or ensembles \citep{jiang2021assessing,chen2021detecting,ng2023predicting}, or aligning model accuracy in-distribution with a confidence threshold \citep{guillory2021predicting,garg2021leveraging,yu2022predicting}. Most of these methods are nontrivial to extend to generation use-cases of LLMs; some require training multiple models, and some need well-defined in-distribution data related to the new task.

\paragraph{Routing LLMs} Prior work on selecting LLMs primarily considers choosing one that produces the best generation for a given input. 
\citet{liu2021simcls,ravaut2022summareranker,jiang2023llm} train dedicated scoring or ranking models that can be applied to model generations. Unlike our work, these approaches require generating outputs with \emph{every} candidate LLM to make a decision, which can be computationally prohibitive with a large pool of candidate LLMs. FrugalGPT \citep{chen2023frugalgpt} calls LLMs sequentially until a dedicated scoring model deems the generation acceptable. Prior works in this group require training data sufficiently representative of each of the tasks and domains of interest to train the corresponding ranking and scoring models. In this paper, instead, we use data from benchmarks to learn the strengths and weaknesses of LLMs across tasks and domains. The resulting model router requires generating outputs only with the chosen LLM at test time.

\section{Learning from Benchmarks}

We start by introducing notation to describe the majority of NLP benchmarks. Let $\{x^d_1,\dots,x^d_{n_d}\}_{d=1}^D$ be a collection of inputs across $D$ tasks. Each input text $x^d_i$ corresponds to a reference answer $r^d_i$, i.e., an ideal generation for the corresponding input. Finally, there is a metric $F_d(x,o,r)$ that can be task-dependent and measures how well a response $o$ for an input $x$ corresponds to the reference $r$. To test an LLM$_m$, $m\in\{1,\dots,M\}$, on the benchmark, for each task $d=1,\dots,D$, its responses are generated $\{o^d_{im} = \text{LLM}_m(x^d_i)\}_{i=1}^{n_d}$ and compared to the corresponding references to obtain performance metrics $\{f^d_{im} = F_d(x^d_i, o^d_{im}, r^d_i)\}_{i=1}^{n_d}$.\footnote{We omit dependency on the prompt when generating with an LLM and, w.l.o.g., consider the same LLM with a different prompting strategy as a different LLM.} At this point, the majority of the benchmark studies will take a (weighted) average of the performance metrics and report a single score for every LLM to rank them in performance. Instead, we reuse these evaluation results to formulate a supervised learning problem to better understand the strengths and weaknesses of various LLMs based on their performance on data points and tasks.

\paragraph{Supervised learning from benchmarks} Our goal is to learn a simple routing function $g_m(x)$ for each LLM, $m=1,\dots,M$, that can predict $\{f^{d'}_{im}\}_{i=1}^{n_{d'}}$, i.e., the performance of the corresponding LLM on a new task $d'$. Then it is trivial to select the best LLM for this task. For efficiency at test time, we restrict the routers $\{g_m\}_{m=1}^M$ to only depend on the input $x$. This is in contrast to the majority of prior works on LLM routing that first obtain generations with every candidate LLM and then use them to choose the best model \citep{liu2021simcls,ravaut2022summareranker,jiang2023llm}. With thousands of open-source LLMs, it is simply infeasible to obtain generations with every LLM for every input at test time.

To complete the problem formulation, we denote the ``correctness'' of model $m$ on an input $x$ by $y(x, m) \in \{0,1\}$. Correctness is evaluated as follows: generate a response $o^d_{im}$ with LLM $m$ on input $x^d_i$, compare it to the corresponding reference $r^d_i$, and output 1 if the model's response is good enough, i.e., $f^d_{im}>\eta_d$, and 0 otherwise, where $\eta_d$ is some threshold that can be task and/or metric specific. For tasks like classification or multiple-choice QA, $y(x^d_i, m) = f^d_{im}$, while for various evaluation metrics used in summarization and instruction following tasks \citep{zhang2020bertscore,sellam2020bleurt,yuan2021bartscore}, the notion of correctness can help to account for the heterogeneity of popular metrics and task difficulty levels. In Section \ref{sec:mix-instruct}, we also present results with raw metrics instead of correctness.

To train a predictor of an LLM correctness, for each LLM, $m=1,\dots,M$, we solve the following optimization problem:
\begin{equation}
    \label{eq:predictor}
    \min_{g_m} \sum_{d=1}^D \sum_{i=1}^{n_d} \ell(g_m(x^d_i), y(x^d_i, m)),
\end{equation}
where we choose $\ell$ to be a binary cross-entropy loss and $g_m$ is any standard probabilistic classifier, i.e., $g_m(x)$ estimates $P(y(x,m)=1|x)$.

An important consideration when training correctness predictors is their ability to generalize out-of-distribution (OOD) data, since our goal is to estimate LLM performance on a new task $d'$ that has not been seen during training. Training predictors given data from multiple domains that need to generalize to unseen domains is indeed an active area of research in ML literature. For example, \citet{sun2016deep,arjovsky2019invariant} proposed methods for improving OOD generalization when training on data from multiple domains, while \citet{koh2021wilds} proposed a benchmark for OOD generalization demonstrating the challenging nature of the problem in various applications.

In this work, we use a simple model for the correctness predictor: we embed all inputs with a sentence transformer \citep{reimers-2019-sentence-bert} and use a $k$-nearest neighbors classifier \citep{cover1967nearest} as $\{g_m\}_{m=1}^M$. kNN is a simple non-parametric classifier that allows us to fit a potentially complicated decision boundary of an LLM's correctness across multiple tasks without extensive hyperparameter tuning. We choose this approach for learning correctness predictors to emphasize the utility of learning from benchmarks even with a basic method and instead focus on the question specific to our problem that has not been studied in prior works on OOD generalization: \emph{Can we improve the quality of LLM routing with an imperfect correctness predictor}?

\section{LLM routing with (imperfect) correctness predictors}
\label{sec:scores}

The goal of LLM routing is to identify an LLM that will have the highest frequency of being correct on a new task $d'$, given the inputs $\{x^{d'}_i\}_{i=1}^{n_{d'}}$ from this task:
\begin{equation}\textstyle
\label{eq:routing}
    \argmax_m \tilde S (m, d'),\,\textrm{where}\ \tilde S (m, d') = \frac{1}{n^{d'}}\sum_{i=1}^{n_{d'}} y(x^{d'}_i, m).
\end{equation}
Here, $\tilde S (m, d')$ is the ``oracle'' score that we want to estimate. The most intuitive estimator is simply using the correctness predictor\begin{equation}\textstyle
\label{eq:score_prob}
S_1(m, d') = \frac{1}{n^{d'}}\sum_{i=1}^{n_{d'}} g_m(x^{d'}_i),
\end{equation}
but prior work has shown that accurately estimating $P(y|x)$, i.e., calibration, is challenging on OOD data \citep{ovadia2019can}. Meanwhile, $g_m$ may still produce accurate predictions after thresholding the predicted probability even if the class probabilities are not estimated well, which is often the case with neural networks \citep{guo2017calibration}. This motivates another score:
\begin{equation}\textstyle
\label{eq:score_acc}
S_2(m, d') = \frac{1}{n^{d'}}\sum_{i=1}^{n_{d'}} \bar g_m(x^{d'}_i)\text{, where } \bar g_m(x^{d'}_i) = \mathbb{I}(g_m(x^{d'}_i) > t),
\end{equation}
where $t\in (0,1)$ is some threshold, e.g., $t=0.5$, $\mathbb{I}$ is an indicator function, and $\bar g_m (x) \in \{0,1\}$ can be interpreted as the prediction of $g_m$ on $x$. 
This score, however, does not take into account the potential ``imperfection'' of $g_m$, i.e., lower accuracy on OOD data from task $d'$. To address this issue, we model the out-of-distribution confidence of the predictions $\bar g_m$.

\paragraph{A simple OOD confidence model} We model LLM correctness as follows:
\begin{equation}
y(x,m)|x, d' =
    \begin{cases}
         \bar g_m(x) & \text{with probability } p(d',m) \\
         1 - \bar g_m(x) & \text{with probability } 1 - p(d',m),
    \end{cases}
\end{equation}
i.e., $p(d',m)\in [0,1]$ is the probability that $\bar g_m$ is the correct prediction on a data point from task $d'$. The above model can be condensed as follows:
\begin{equation}
\label{eq:model_conf}
    y(x,m)|x,d' \sim \text{Bern}(\bar g_m(x)p(d',m) + (1-\bar g_m(x))(1-p(d',m))).
\end{equation}
In this simplistic (and approximate) model, we assume that $p(d',m)$ does not depend on the input $x$ after conditioning on the task $d'$. The assumption is analogous to the homoscedastic error term assumption in linear regression models and allows us to interpret $p(d',m)$ as the marginal/overall accuracy of $\bar g_m$ on data from the task $d'$.

Prior work has studied the problem of estimating OOD accuracy given the inputs from a new task, but existing methods are challenging to combine with our approach. For example, \citet{garg2021leveraging} learn a threshold on model confidence, which is hard to apply when using kNN classifiers, and \citet{ng2023predicting} require data augmentations that can be challenging to identify given the diversity of tasks in benchmarks. Prior methods also do not take into account the partition of the train data into tasks inherent in our problem setting.

We treat the problem of estimating $p(d',m)$ as a supervised learning task, taking advantage of the task partition. Specifically, we assign a task descriptor $u(d) \in \mathbb{R}_{+}$ to every task that measures the distance of the data from task $d$ to the other available tasks combined. Then we collect the values of $p(d,m)$, i.e., the accuracy of $\bar g_m$ on $d$, and fit a non-parametric regression model to predict $p(d,m)$ from $u(d)$. At test time, we compute $u(d')$ for a new task $d'$ based on the inputs $\{x^{d'}_i\}_{i=1}^{n_{d'}}$ and predict $p(d',m)$ using the fitted regression model. In general, one can consider more sophisticated, higher-dimensional task descriptors $u(d)$, but here, for simplicity, we keep it 1-dimensional and use a Gaussian kernel smoother (also known as the Nadaraya-Watson estimator) as the non-parametric regressor. We provide details in Appendix \ref{sup:data-distance}.

Finally, given the model of LLM correctness \ref{eq:model_conf}, $\tilde {\mathbf{S}} (m, d')$ is a random variable (corresponding to $\tilde S (m, d')$) distributed as a (scaled) sum of two Bernoulli random variables. To arrive at our final score for LLM routing, we take its expected value:
\begin{equation}
\label{eq:score_conf}
    S_3(m, d') = S_2(m, d')p(d', m) + (1 - S_2(m, d'))(1 - p(d', m)).
\end{equation}

When selecting an LLM with $S_3$, we consider an alternative to the $\argmax$ criterion based on our correctness model \ref{eq:model_conf}, which defaults to the best model on average across benchmark datasets when we are not sufficiently confident that a candidate model will be better: 
\begin{equation}
\label{eq:select_s3}
    \begin{cases}
    m_3 & \text{if }P(\tilde{\mathbf{S}}(m_3, d') > \tilde{\mathbf{S}}(m^*, d')) > \eta \\
    m^* & \text{otherwise},
    \end{cases}
\end{equation}
where $m_3 = \argmax_m S_3(m,d')$, i.e., the best LLM for the new task according to $S_3$, and $m^* = \argmax_m \sum_{d=1}^D \tilde S (m, d)$, i.e., the best LLM across the benchmark datasets.
In the experiments, we set $\eta = 0.6$.

We summarize our LLM routing procedures in Appendix \ref{sup:data-distance}.

\subsection{Connection to meta-learning}

The OOD confidence model in \eqref{eq:model_conf} is a meta-model of routing across multiple tasks, and fitting it entails a form of meta-learning. Consider the meta-learning problem 
\begin{equation}\textstyle
\min_{g_m, p(\cdot,m)} \sum_{d=1}^D \sum_{i=1}^{n_d} \ell(\bar{g}_m(x_i^d)p(d,m) + (1-\bar{g}_m(x_i^d))(1-p(d,m)), y(x^d_i, m)),
\label{eq:meta-predictor}
\end{equation}
where $\bar g_m$ and $p(\cdot,m)$ are meta-parameters and adaptation step $\bar g_m\to \bar g_m(\cdot) p(\cdot,m)$ adaptively shrinks the router output towards ambiguity. We exploit this connection to theoretically demonstrate the potential advantages of routing LLMs using $S_3$ over $S_2$.

In expectation/in the population, \eqref{eq:meta-predictor} fits a larger model class than \eqref{eq:predictor}, so the risk of the adaptively shrunken router is at most that of the non-adaptive router:
\begin{equation}
\begin{aligned}
&\textstyle\sum_{d=1}^D\mathbf{E}\big[\ell(\bar{g}_m(X^d)p(d,m) + (1-\bar{g}_m(X^d))(1-p(d,m)), y(X^d, m))\big] \\
&\textstyle\quad\le \sum_{d=1}^D\mathbf{E}\big[\ell(\bar{g}_m(X^d), y(X^d, m))\big].
\end{aligned}
\label{eq:adaptivity-improvement}
\end{equation}
This suggests (subject to standard assumptions on the loss function) that adaptive shrinkage routing leads to better approximations of the oracle router. Lemma \ref{lem:adaptive-improvement} confirms this intuition.

\begin{lemma}
\label{lem:adaptive-improvement}
Let $\ell(y_1,y_2) = \rho(y_1 - y_2)$ for some subadditive $\rho:\mathbf{R}\to\mathbf{R}$ (e.g.\ $\rho(x) = \frac12x^2$ for the square loss). We have 
\[
\begin{aligned}
\ell(S_2,\widetilde{S}) &\le \mathbf{E}\big[\ell(\bar{g}_m(X^d),y(X^d,m))\big]), \\
\ell(S_3,\widetilde{S}) &\le \mathbf{E}\big[\ell(p(d,m)\bar{g}_m(X^d) + (1-p(d,m))(1-\bar{g}_m(X^d)),y(X^d,m))\big]).
\end{aligned}
\]
\end{lemma}

We present the proof in Appendix \ref{sup:lemma}. Combining \eqref{eq:adaptivity-improvement} and Lemma \ref{lem:adaptive-improvement}, we expect the adaptive router based on $S_3$ to outperform its non-adaptive counterpart based on $S_2$. That said, it is unclear whether adaptive shrinkage will improve the performance of the adaptive router in finite samples: the expected performance of the adaptive router may be offset by the inflation in variance from fitting the larger (adaptive) model class. Fortunately, our empirical results show that task-specific adaption, i.e., using $S_3$ as a score for routing, generally improves performance. The two-step method for fitting $\bar g_m$ and $p$ in Section \ref{sec:scores} approximately minimizes \eqref{eq:meta-predictor} with a single Gauss-Seidel pass through the decision variables.

\section{Experiments}

\subsection{Model routing on HELM}
\label{sec:helm}

We explore the benefits and challenges of learning from benchmarks using the HELM \citep{liang2022holistic} benchmark.

\paragraph{Data} We select 29 datasets representing scenarios such as question answering (including a subset of MMLU \citep{hendrycks2020measuring}), text classification, language, knowledge, and reasoning, among others. We present additional information about these datasets in Table \ref{tab:helm_dataset_desc}.

\paragraph{Models} We evaluate 18 open-source models ranging in size from 3B to 70B, including base and chat variations of Llama 2 in different sizes. All models are summarized in Table \ref{tab:models}.

\paragraph{Model routing} The best model on average (BMA) across the 29 considered HELM datasets is \texttt{llama-2-70b} (followed by \texttt{llama-2-70b-chat}). Our goal is to show that learning model routers from benchmark data can simultaneously outperform BMA and reduce inference costs by recommending smaller LLMs for tasks where they can perform well. We compare models selected with the three scores, $S_1, S_2$, and $S_3$, presented in Section \ref{sec:scores} to the performance of \texttt{llama-2-70b}, i.e., the BMA. All correctness predictors $g_m$s are kNN classifiers with $k=5$.

We also report the performance of the best model according to the ``oracle'' score $\tilde S$, which is the upper bound on what can be achieved with model routing, and $\tilde S_3$, which corresponds to $S_3$ with the true $p(d', m)$, i.e., the accuracy of (an imperfect) $g_m$ on $d'$. Finally, we compare to scoring LLMs with the average log-likelihood (LL) (or negative perplexity) of the response they generate on the inputs from the task of interest. 
This last baseline requires producing generations with \emph{every} LLM at test time to make a selection, while all of our scores only require generating with the chosen LLM.

\begin{table}[t]
    \caption{LLM routing on HELM: Comparison of various model scores for LLM routing with the Oracle model selection and performance of the best model on average (BMA).}
    \vspace{-2mm}
    \label{tab:helm}
    \centering
\begin{tabular}{lccccccc}
\toprule
 & Acc. & Ratio to Best & Pearson & Spearman & \% BMA & \# Params & Rank \\
\midrule
$S_1$ eq. \ref{eq:score_prob} & 0.662 & 0.855 & 0.685 & 0.465 & 0.17 & \underline{40.3B} & 6.172 \\
$S_2$ eq. \ref{eq:score_acc} & 0.676 & 0.868 & 0.636 & 0.468 & 0.10 & 44.3B & 5.897 \\
$S_3$ eq. \ref{eq:score_conf}, \ref{eq:select_s3} & \underline{0.694} & \underline{0.898} & \underline{0.727} & \underline{0.492} & 0.48 & 49.8B & \underline{5.310} \\
$S_3$ true $p$ & \textbf{0.735} & \textbf{0.944} & \textbf{0.799} & \textbf{0.596} & 0.22 & \textbf{33.8B} & \textbf{3.800} \\
LL & 0.684 & 0.869 & 0.714 & 0.459 & 0.10 & --- & 6.517 \\
BMA & 0.688 & 0.884 & --- & --- & 1.00 & 70.0B & 6.069 \\
\midrule
Oracle & 0.773 & 1.000 & --- & --- & 0.21 & 29.1B & 1.000 \\
\bottomrule
\end{tabular}
\end{table}

\paragraph{Results} We conduct 29 sets of experiments, each time selecting 28 of the datasets as the benchmark data for training the LLM routers and using the remaining task as the new task $d'$ for evaluating the quality of the LLM selection for this task. In Table \ref{tab:helm} we report averages across experiments for the performance of the selected model (Acc.), ratio of this performance to the performance of the best model for the corresponding new task (Ratio to Best), Pearson and Spearman rank correlations between model accuracies and model scores, number of parameters of the selected model (\# Params), rank of the selected model out of 18 considered (Rank). We also report the fraction of times the BMA is selected by a method (\% BMA). Best results are highlighted with bold and second best with an underline (excluding Oracle).

First, we notice that accounting for imperfections of the correctness predictors (their average accuracy is 0.59) has clear benefits: when we have access to the true accuracy of correctness predictors, the corresponding score, $S_3$ true $p$, noticeably outperforms all other scores. Our simple kernel smoothing estimator of this accuracy (MAE$=0.116$) allows us to obtain a practical model routing score $S_3$ that outperforms BMA (\texttt{llama-2-70b}) while choosing smaller models for some of the tasks (as evident by the average number of parameters of the chosen models). $S_2$ sacrifices some accuracy but chooses even smaller performant models. Overall, learning from benchmarks allows us to obtain LLM routers that can improve overall performance while utilizing smaller models where appropriate. Finally, we note that log-likelihood (LL) also performs well, however, routing with it requires passing each test input through \emph{each} candidate LLM, which have 347B parameters in total.

\begin{wrapfigure}[16]{r}{0.45\linewidth}
\vspace{-0.2in}
    \centering
    \includegraphics[width=\linewidth]{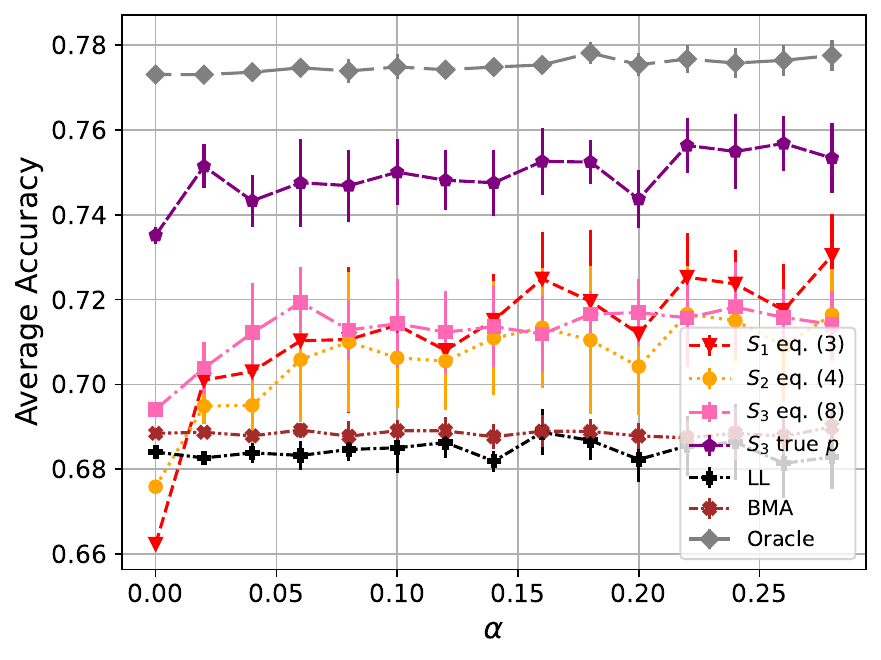}
    \vspace{-0.32in}
    \caption{Using $\min(\alpha n^{d'}, 50)$ training samples from $d'$ to reduce OOD gap.}
    \label{fig:helm-ood}
\end{wrapfigure}

\paragraph{Reducing the OOD gap} The average accuracy of correctness predictors across tasks and models for the experiments in Table \ref{tab:helm} is 0.59. It is a fairly low accuracy for binary classification, which we attribute to the diversity of tasks in the HELM benchmark leading to substantial distribution shifts when predicting the correctness of LLMs on held-out tasks. We investigate the quality of model routing when we reduce this OOD gap. A simple strategy to reduce this gap is to collect a small number of labeled in-distribution samples. This can be accomplished by asking a practitioner to provide reference answers ($r_i^{d'}$s) for a small number of inputs from their task, allowing us to evaluate the correctness of candidate LLMs on these in-distribution inputs and use them to improve correctness predictors.

We simulate this scenario by moving $\min(\alpha n^{d'}, 50)$ samples from the data from a new task $d'$ to the data for training the correctness predictors. The upper limit of 50 samples is to maintain practical utility while accounting for varying dataset sizes (see Table \ref{tab:helm_dataset_desc}). We conduct 29 sets of experiments, repeating each one 10 times to obtain standard deviations (randomness is due to random selection of data points from a new task for reducing the OOD gap). We summarize the average accuracy of models selected with various routing scores for varying $\alpha$ in Figure \ref{fig:helm-ood} ($\alpha=0$ corresponds to Table \ref{tab:helm}). Results for Pearson correlation are in Figure \ref{fig:helm-ood-sup}(a).

We see that even a small number of in-distribution samples ($\alpha=0.05$) can reduce the OOD gap (corresponding average accuracy of correctness predictors is 0.65; see Figure \ref{fig:helm-ood-sup}(b)) and noticeably improves the model routing performance of all three of our scores. When the number of in-distribution samples further increases, $S_1$ starts to outperform $S_3$. We attribute this observation to kNN being well-calibrated in-distribution, i.e., the correctness predictors provide reliable estimates of their own confidence $P(y|x)$, which are used by $S_1$ in \eqref{eq:score_prob}.  
Finally, we note a fairly large variance in the results due to random selection of the in-distribution training samples from $d'$, suggesting that active learning \citep{settles2009active} can help to further improve LLM routing.

\subsection{Model Routing on Mix-Instruct}
\label{sec:mix-instruct}

\begin{wrapfigure}[18]{r}{0.5\linewidth}
\vspace{-0.22in}
    \centering
    \includegraphics[width=\linewidth]{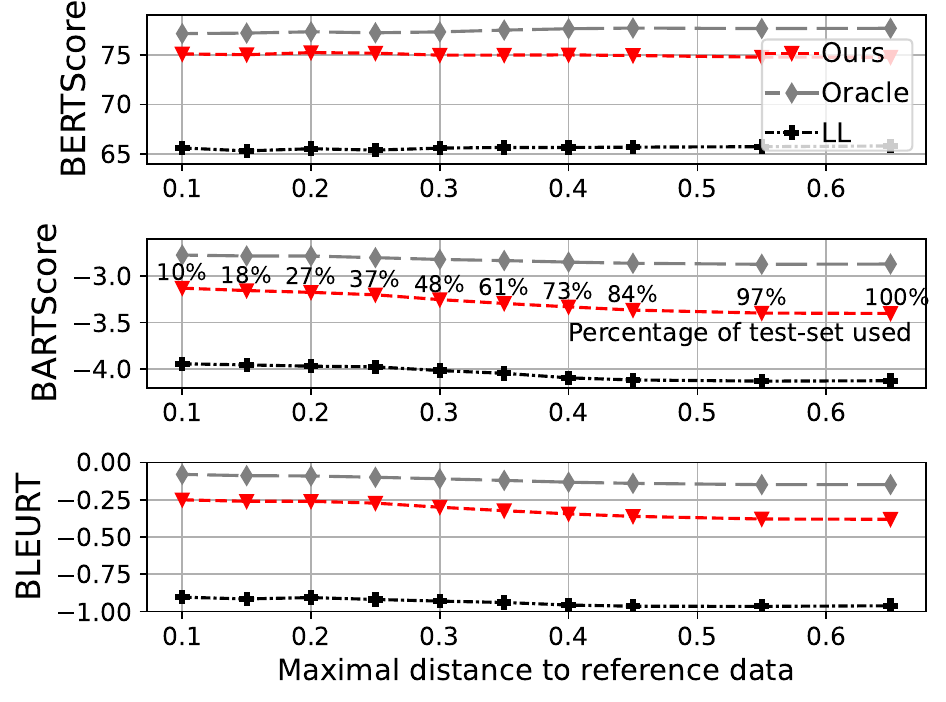}
    \vspace{-0.3in}
    \caption{Average metrics on subsets of the MixInstruct test set, defined by limiting the maximal average distance between test instances and their closest neighbors in the reference (train) set.} 
    \label{fig:mixinstruct_per_dist_thresh}
\end{wrapfigure}

We now consider a different setting and task type, the MixInstruct benchmark dataset \citep{jiang2023llm}.
The dataset is composed of instruction-following tasks, divided into train/validation/test sets of 100K/5K/5K samples, and includes evaluations of $N=11$ open-source LLMs using common metrics, e.g. BERTScore \citep{zhang2020bertscore}, BARTScore \citep{yuan2021bartscore}, and BLEURT \citep{sellam2020bleurt}.
In \citet{jiang2023llm}, this benchmark was used to compare different LLM ranking methods in per-instance model selection.
We follow the same setting and apply our score $S_1(m,d')$ to the test set, per-instance, where we use the 100K-sample train set as the benchmark data for training our LLM router. 
See Appendices \ref{sup:data-distance} and \ref{sup:mix-instruct} for details on the score computation and the experiment parameters, respectively. 
Due to the per-instance setting, and since the test set was constructed from in-distribution data, we focus on our simplest router model $S_1$, \eqref{eq:score_prob}.

We compare our approach with the scoring methods examined by \citet{jiang2023llm}, as well as scoring based on the average log-likelihood (LL) of the model responses to the inputs. Additionally, we present the metrics for the best models on average (BMA), Open-Assistant \citep{oasst2023}
and Vicuna \citep{vicuna2023}.
We report the results of BERTScore, BARTScore and BLEURT in Table \ref{tab:mixinstruct_per_point_results}, along with the number of model calls per instance (MCPI) performed during inference time. 
All compared methods require model generations for every point in the test set, by each of the examined LLMs, whereas our approach requires only one model generation and one call to some general embedding function.
In addition, all methods, except for LL, require training auxiliary language models, whereas our approach is a simple kNN classifier on the embedded inputs.
While our approach does not consistently outperform the compared methods, these results demonstrate the potential of using benchmark datasets for model routing with significantly better inference-time efficiency.

\begin{table}[t]
\caption{Average metrics for per-instance LLM selection on the test set of MixInstruct. MCPI denotes model calls per instance, for $N$ models. Best results are highlighted with bold and second best with an underline (excluding Oracle).}
\vspace{-2mm}
\label{tab:mixinstruct_per_point_results}
\centering
\begin{tabular}{lcccc} 
 \toprule
 & BERTScore $\uparrow$ & BARTScore $\uparrow$ & BLEURT $\uparrow$ & MCPI\\ [0.1ex] 
 \midrule
 Random & 66.36 & -3.76 & -0.77 & - \\ 
 LL & 65.83 & -4.12 & -0.96 & $N$ \\
 BMA: Open-Assisant & \underline{74.68} & -3.45 & -0.39 & - \\
 BMA: Vicuna  & 69.60 & -3.44 & -0.61 & - \\
 MLM-Scoring \citep{salazar2020masked} & 64.77 & -4.03 & -0.88 & $N$ \\
 SimCLS \citep{liu2021simcls} & 73.14 & \underline{-3.22} & \underline{-0.38} & $N$ \\
 SummaReranker \citep{ravaut2022summareranker} & 71.60 & -3.25 & -0.41 & $N$ \\
 PairRanker \citep{jiang2023llm} & 72.97 & \textbf{-3.14} & \textbf{-0.37} & $N$ \\
 Ours & \textbf{74.75} & -3.40 & \underline{-0.38} & \textbf{2} \\
 \midrule
 Oracle & 77.67 & -2.87 & -0.15 & $N$ \\
 \bottomrule
\end{tabular}
\end{table}

\paragraph{Effect of benchmark dataset sparsity}
\sloppypar
To highlight the potential of our approach in this setting, we examine the effect of the reference benchmark data sparsity. 
We apply our method to different subsets of the test set, $X_{\mathrm{test}}$, where the subsets are defined by limiting the maximal average distance of each test set point to the closest points from the reference (train) set, denoted by $\mathrm{NN}_{\mathrm{train}}$, i.e. $X'_C = \left\lbrace x'\in X_{\mathrm{test}} \Big| \tfrac{1}{\left\vert \mathrm{NN}_{\mathrm{train}}(x')\right\vert}\sum_{x\in \mathrm{NN}_{\mathrm{train}}(x')}\mathrm{dist}(x',x) < C\right\rbrace$, where $C$ is the maximal average distance and $X'_C$ is the resulting subset of the test set. Figure \ref{fig:mixinstruct_per_dist_thresh} presents the metric scores for the different subsets using our method, the oracle (best possible choices), and LL scoring. We also report the percentage of the test set that is used in each subset.
This figure depicts that our predictor approaches the oracle metrics as the average distance to the reference points decreases. 
This suggests that adding more benchmark datasets, to reduce the sparsity of the reference space, may lead to better LLM selections with our approach.

\section{Discussion and Conclusion}

\begin{wrapfigure}[14]{r}{0.45\linewidth}
\vspace{-0.45in}
    \centering
    \includegraphics[width=\linewidth]{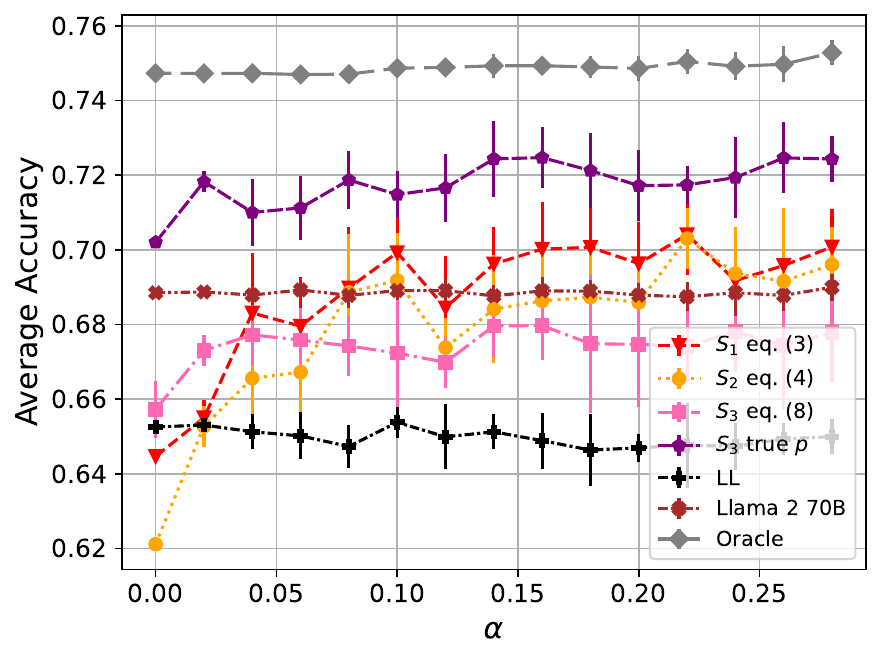}
    \vspace{-0.31in}
    \caption{LLM routing with $\leq13$B parameter models compared to \texttt{Llama 2 70B}.}
    \label{fig:helm-ood-small}
\end{wrapfigure}

\paragraph{How useful are smaller LLMs?} While a given LLM may work best on average, these models tend to be the biggest and therefore most expensive to run.  Practitioners can achieve gains in cost, compute, and latency if we can successfully predict whether a smaller LLM can be adequate for a given task. Identifying good smaller models for tasks of interest will also redefine the cost/benefit tradeoff behind automating certain tasks, potentially incentivizing the automation of new tasks that were previously cost-prohibitive to automate with larger LLMs.

To evaluate the potential of smaller LLMs we revisit our HELM experiment in Figure \ref{fig:helm-ood}. In Figure \ref{fig:helm-ood-small}, we perform LLM routing using \emph{only models with $\leq$ 13B parameters} and compare it to the performance of \texttt{Llama 2 70B}. Oracle's performance demonstrates that it is conceptually possible to outperform a large model by routing smaller LLMs. Results with our scores $S_1$ and $S_2$ (see Figure \ref{fig:helm-13b-individual} for breakdown by scores) demonstrate that it is also practically feasible to match the performance of the 70B model by combining learning from benchmarks with a small number ($\alpha=0.04$, i.e., 2-40 samples) of labeled samples from a new task that a practitioner can provide to save on the inference costs in their LLM application.

\begin{wrapfigure}[16]{r}{0.45\linewidth}
\vspace{-0.15in}
    \centering
    \includegraphics[width=\linewidth]{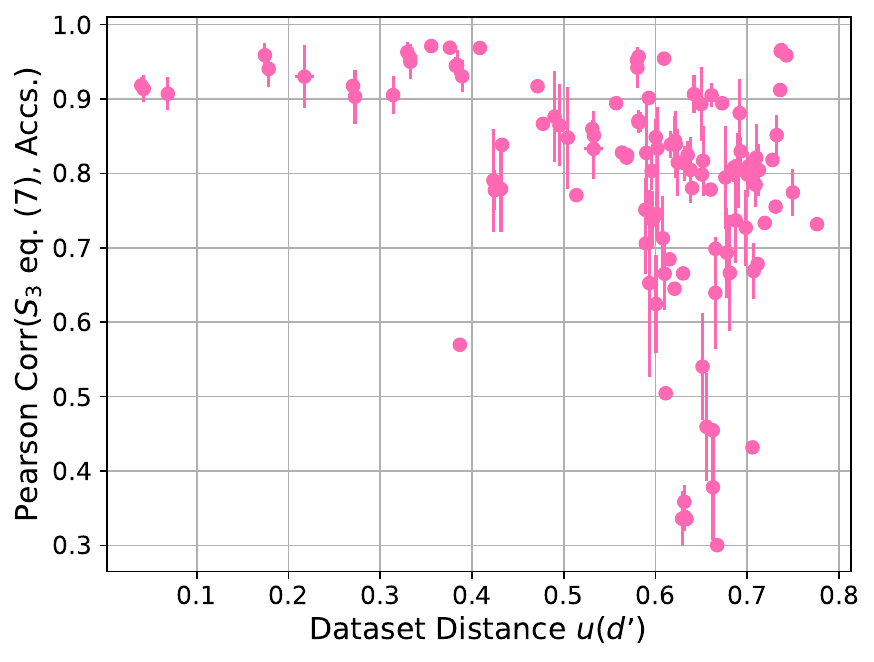}
    \vspace{-0.3in}
    \caption{Correlation($S_3$, Accs.) and $u(d')$.}
    \label{fig:distance-corr}
\end{wrapfigure}

\paragraph{Learning from more benchmarks} We anticipate learning LLM routers from benchmarks to be the most effective when new tasks are similar to the benchmark tasks, thus reducing the OOD gap without any labeling burden for a practitioner. To empirically investigate this hypothesis, in Figure \ref{fig:distance-corr} we visualize the relation between the quality of model routing with $S_3$, measured with Pearson correlation between model scores and accuracies of candidate LLMs, and the distance $u(d')$ from a new task $d'$ to the available benchmark data for training the routers.
In this experiment, we aggregate results across different $\alpha$ values from Figure \ref{fig:helm-ood}. For smaller distance values the correlation is approaching 1, while for large distances it sometimes deteriorates. Results for other scores demonstrate a similar trend and are presented in Appendix \ref{sup:corr-distance} along with additional details. 
This experiment and the benchmark dataset sparsity analysis presented in Figure \ref{fig:mixinstruct_per_dist_thresh} for MixInstruct illustrate that learning with \emph{more benchmarks} can improve the efficacy and reliability of the LLM routers as new tasks are more likely to be closer to a large collection of datasets.

\paragraph{Future work} Our work demonstrates the potential of learning from benchmarks for LLM routing and investigates 3 model scores in the context of OOD generalization when routing LLMs for new tasks. We summarize potential next steps for improving the quality and efficacy of LLM routers.

The major challenge of LLM routing is OOD generalization of correctness predictors. Thus, using more benchmarks and modern methods for improving OOD generalization to learn correctness predictors is a promising next step. A practitioner can also provide labels for a few samples from their task, possibly guided by active learning techniques, to adapt or fine-tune correctness predictors. Even when reducing the OOD gap is too challenging, our score accounting for the (potentially low) accuracy of correctness predictors demonstrated strong results when this accuracy, $p(d',m)$, is known for a new task, thus encouraging the development of methods for estimating it better.

We also anticipate that routing ``expert'' LLMs fine-tuned for a specific domain can improve the results. Regions of the sample space where such models are ``correct'' should mostly align with the domains of their expertise (recall Figure \ref{fig:diagram}), making it easier to learn the corresponding correctness predictors, and simplifying LLM routing when a new task is from a specific domain. 

Our experiments in Figure \ref{fig:helm-ood-small} demonstrate the utility of LLM routing with \emph{smaller} models, which can reduce costs and facilitate the use of LLMs in a broader set of domains. Thus, we want to explore modifications to our scores that will encourage the selection of smaller LLMs when their anticipated performance is comparable to the larger, more reliable models. Prior work on frugal API selection \citep{chen2020frugalml,chen2023frugalgpt} provides a good starting point to explore this direction.

\bibliography{paper}
\bibliographystyle{apalike}

\newpage
\appendix

\section{Correctness predictors and confidence estimation}
\label{sup:data-distance}

We provide additional details on the correctness predictors used in our experiments, along with more details on the dataset distance and the Gaussian kernel smoother for estimating the accuracy of the correctness predictors on new tasks, $p(d',m)$s. 

\paragraph{Correctness predictors in our experiments}
While any probabilistic classifier may fit our setting, in the experiments we mainly used a simple kNN classifier, applied in an embedded space. 
Recall that we have $D$ benchmark datasets with inputs $\{x^d_i\}_{i=1}^{n_d}$ for $d=1,\dots,D$. To compute our correctness predictor based on the benchmark datasets, we first embed all their inputs. We denote the combined set of embedded inputs from the benchmark datasets as $\mathcal{D} = \{\phi(x^d_1),\dots,\phi(x^d_{n_d})\}_{d=1}^D$, where $\phi$ is a sentence transformer \citep{reimers-2019-sentence-bert}. We use \texttt{all-mpnet-base-v2} from Hugging Face in all experiments.
Given a sample $x_i^{d'}$ from a new task $d'$, we embed it using the same $\phi$ and define the classifier, $g_m$, for each model $m$ by:
\begin{equation}
    g_m\left(x_i^{d'}\right) = \frac{1}{k}\sum_{e\in \text{NN}(\phi(x_i^{d'}),k,\mathcal{D})} y(e,m),
    \label{eq:knn_classifier}
\end{equation}
where $y(e,m)\in\{0,1\}$ is the correctness of model $m$ on the (embedded) input $e$, and $\text{NN}(\phi(x_i^{d'}),k,\mathcal{D})$ is the set of $k$ closest embedded neighbors from $\mathcal{D}$ to the new embedded sample $\phi(x_i^{d'})$, according to the cosine distance.
Then, $\bar g_m$, as defined in \eqref{eq:score_acc}, is a binary kNN classifier.
Finally, we compute the per-model correctness predictors, $S_1(m,d')$ and $S_2(m,d')$, for the new task $d'$, according to \eqref{eq:score_prob} and \eqref{eq:score_acc}, respectively.

Next, we describe a method for estimating the probability $p(d',m)$ in our confidence model and the $S_3(m,d')$ score, \eqref{eq:score_conf}. 
This method comprises a dataset distance and a kernel smoother, defined as follows.

\paragraph{Dataset distance}
Our dataset distance $u(d)$ is a one-sided variant of the Chamfer distance with extended neighborhood size. We define it formally below:
\begin{equation}
\label{eq:data-distance}
    u(d) = \frac{1}{n_{d}} \sum_{i=1}^{n_{d}} nn(x^{d}_i,\mathcal{D}_{-d}),
\end{equation}
where $\mathcal{D}_{-d}$ is the set of (embedded) inputs from the $D$ datasets excluding inputs from $d$ (for a new task $d'$, $\mathcal{D}_{-d'} = \mathcal{D}$ since $d'$ is not part of the $D$ benchmark datasets we use for training LLM routers), and $nn(x^{d}_i,\mathcal{D}_{-d})$ is the average distance from the input $x^{d}_i$ to its closest $\kappa$ neighbors in $\mathcal{D}_{-d}$:
\begin{equation}
\label{eq:point-distance}
 nn(x,\mathcal{D}) = \frac{1}{\kappa} \sum_{e \in \text{NN}(\phi(x), \kappa, \mathcal{D})} \text{cosine}(\phi(x),e),
\end{equation}
where $\text{NN}(\phi(x), \kappa, \mathcal{D})$ is the set of $\kappa$ closest embedded neighbors of $\phi(x)$ in $\mathcal{D}$ according to cosine distance. We set $\kappa=19$ for the dataset distance in all experiments.

\paragraph{Kernel smoother} For each LLM $m=1,\dots,M$, to obtain the corresponding kernel smoother estimate we iterate over the available benchmark datasets, each time holding one out and computing pairs $(u(d), p(d,m))$ for held out dataset $d$, where $p(d,m)$ is the accuracy of $g_m$ on data from $d$ after training on $\mathcal{D}_{-d}$. We repeat this process 10 times for 15 values of in-distribution mixing parameter $\alpha$ (similar to the experimental setup in Figure \ref{fig:helm-ood} but using benchmark datasets $d=1,\dots,D$ instead of $d'$) to obtain the training set of distance-accuracy pairs $\{u_z, p_z(m)\}_{z=1}^Z$. 
In the HELM experiments in Section \ref{sec:helm}, $Z=28*10*15=4200$ (28 is the number of datasets from HELM after holding one out as the new task for evaluating the performance). 

For a new task $d'$, we compute $u(d')$ using the inputs from this task and our benchmark datasets and estimate $p(d',m)$ for each $m$ with simple Gaussian kernel smoothing:
\begin{equation}
\label{eq:kernel_smoother}
    p(d', m) = \frac{\sum_{z=1}^Z p_z(m) \mathcal{K}(u(d'), u_z)}{\sum_{z=1}^Z \mathcal{K}(u(d'), u_z)},
\end{equation}
where $\mathcal{K}(u(d'), u_z) = \exp\left(-\frac{(u(d') - u_z)^2}{2\sigma^2}\right)$. We set $\sigma=0.09$ in all experiments which is the value we found to perform well through some preliminary experimentation.

Finally, we note that the proposed confidence model, including the definitions of the dataset distance and kernel smoother, can be combined with any classifier $g_m$, and is not restricted to the kNN classifier used for the correctness predictor in our experiments.

\paragraph{Additional notes regarding $S_3$}
Recall that when selecting a model with $S_3(m,d')$ we use an additional step described in \eqref{eq:select_s3} that facilitates the selection of the best model on average when we are not sufficiently confident in the model with the highest $S_3(m,d')$ score. Probability expression, $P(\tilde{\mathbf{S}}(m_3, d') > \tilde{\mathbf{S}}(m^*, d'))$, required for this step is not available in closed form, as $\tilde{\mathbf{S}}$ is distributed as a (scaled) sum of two Bernoulli random variables, but it is straightforward to estimate via Monte Carlo sampling from the corresponding Bernoulli distributions.

When reporting correlations for $S_3$ (e.g., Pearson and Spearman correlations in Table \ref{tab:helm}), we use $S_3(m,d')$ as is, i.e., as defined in \eqref{eq:score_conf}.

\section{Additional results for model routing on HELM}



\begin{figure*}
\begin{center}
\subfigure[Pearson correlation]{\includegraphics[width=0.32\textwidth]{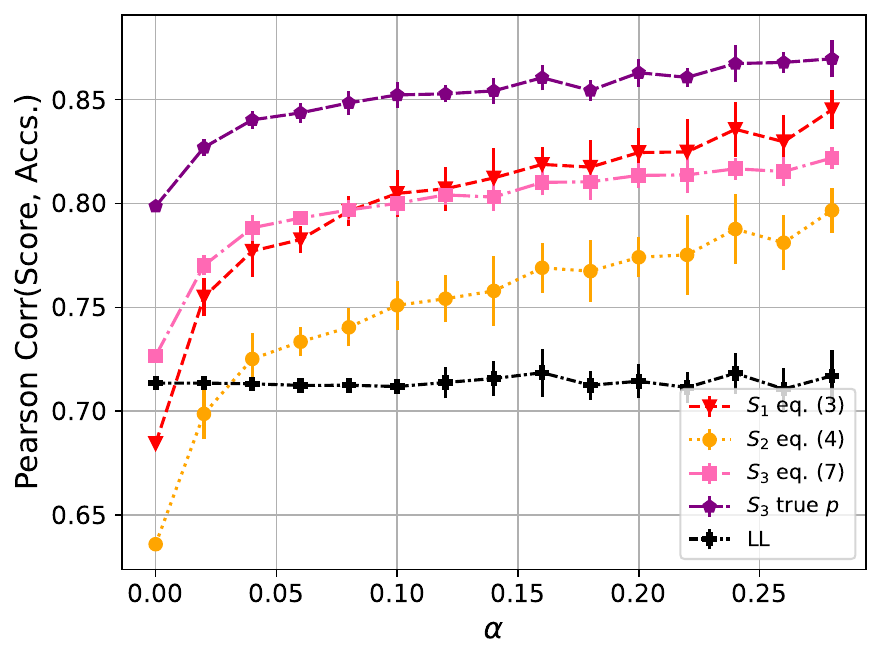}}
\subfigure[Accuracy of $g_m$s]{\includegraphics[width=0.32\textwidth]{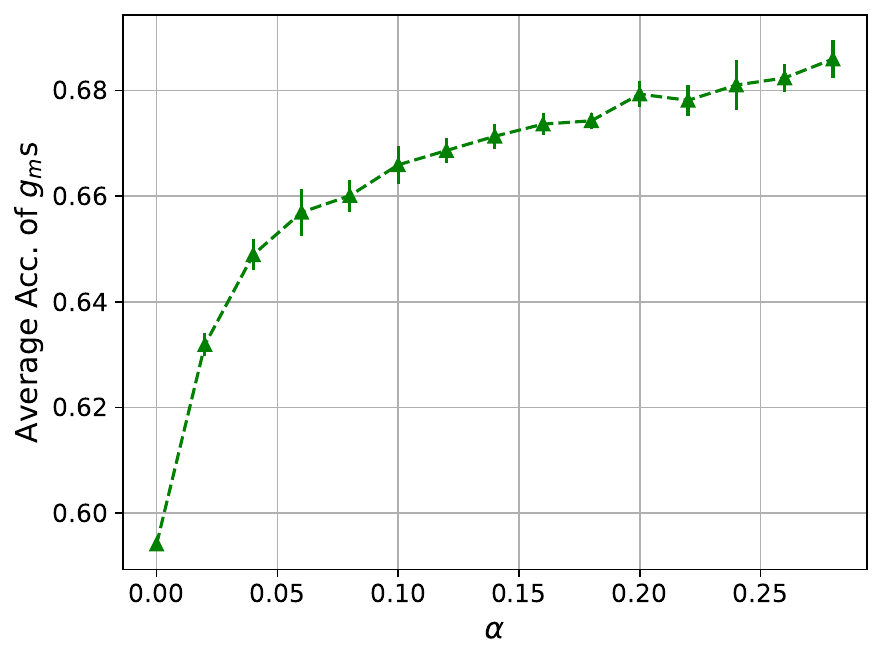}}
\subfigure[Estimation of $p(d', m)$s]{\includegraphics[width=0.32\textwidth]{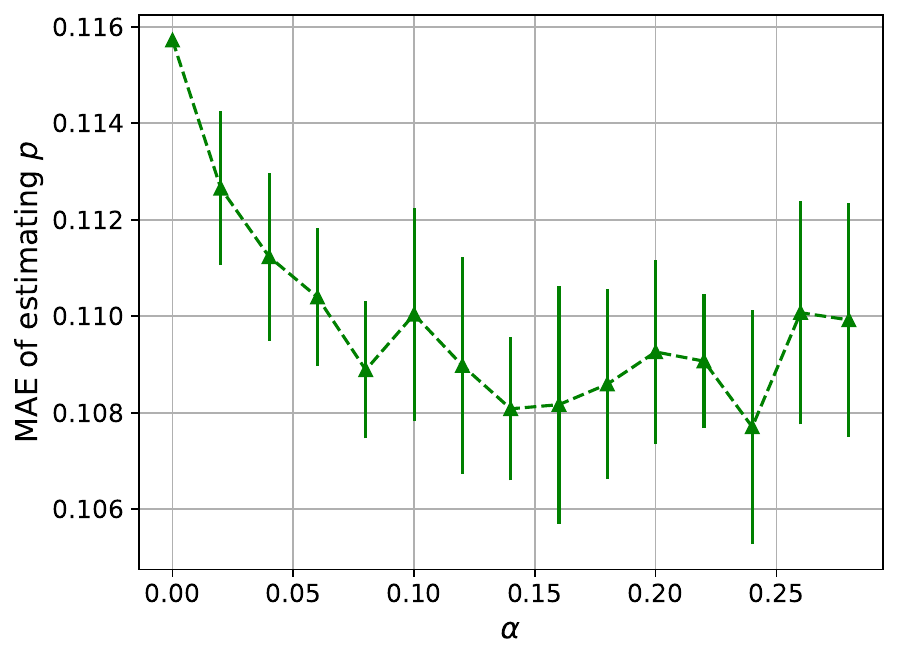}}
\vskip -0.1in
\caption{Additional results for Reducing the OOD gap experiment in Figure \ref{fig:helm-ood}.}
\label{fig:helm-ood-sup}
\end{center}
\vskip -0.2in
\end{figure*}

\subsection{Reducing the OOD gap} 
We present additional results for this experiment in Figure \ref{fig:helm-ood-sup}. (a) shows Pearson correlation improvement as we increase $\alpha$, similar to the trends in accuracy improvement in Figure \ref{fig:helm-ood}; (b) demonstrates that the accuracy of correctness predictors $g_m$s improves as we increase the number of samples from $d'$ used for training them, thus reducing the OOD gap; (c) shows the mean absolute error (MAE) of our kernel smoothing estimator of the accuracy of correctness predictors $p(d', m)$ -- the estimator does not improve as much with increased $\alpha$, thus $S_3$ eventually becomes worse than $S_1$ in terms of correlation and accuracy of the selected models.

\begin{figure*}
\begin{center}
\subfigure[$S_1$ \eqref{eq:score_prob}]{\includegraphics[width=0.32\textwidth]{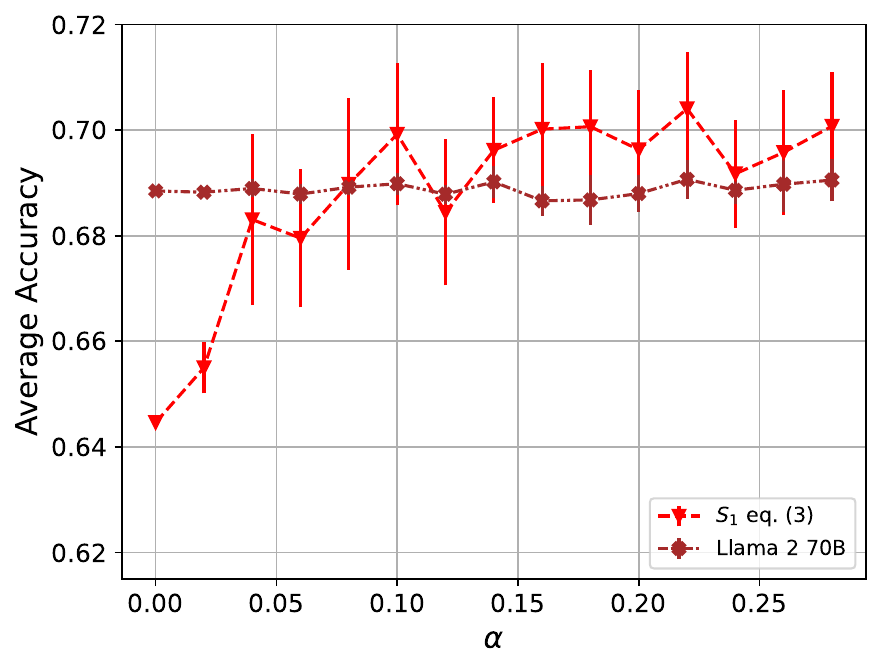}}
\subfigure[$S_2$ \eqref{eq:score_acc}]{\includegraphics[width=0.32\textwidth]{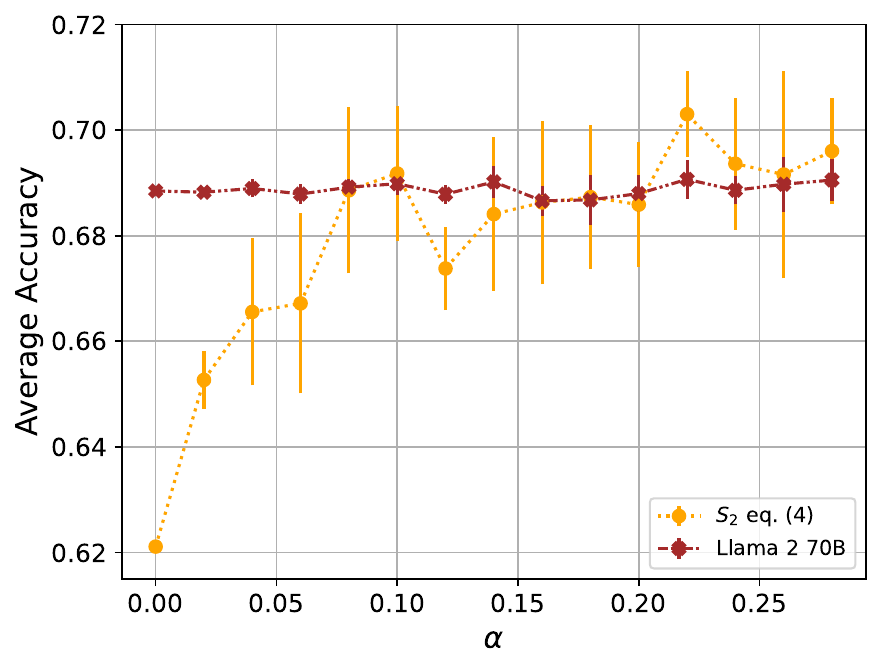}}
\subfigure[$S_3$ \eqref{eq:select_s3}]{\includegraphics[width=0.32\textwidth]{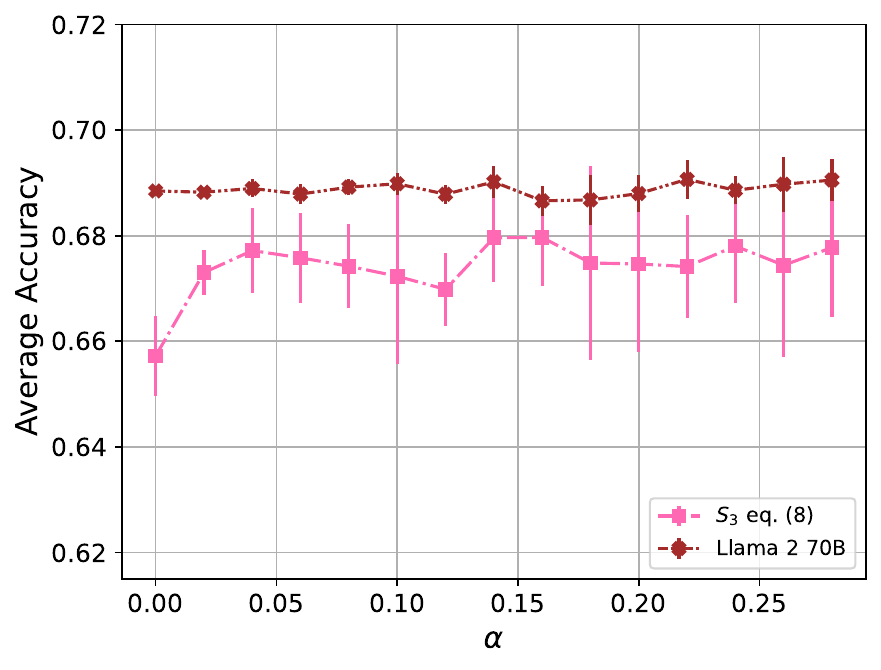}}
\vskip -0.1in
\caption{LLM routing with $\leq13$B parameter models compared to \texttt{Llama 2 70B}.}
\label{fig:helm-13b-individual}
\end{center}
\vskip -0.2in
\end{figure*}

\begin{figure*}
\begin{center}
\subfigure[$S_1$ \eqref{eq:score_prob}]{\includegraphics[width=0.32\textwidth]{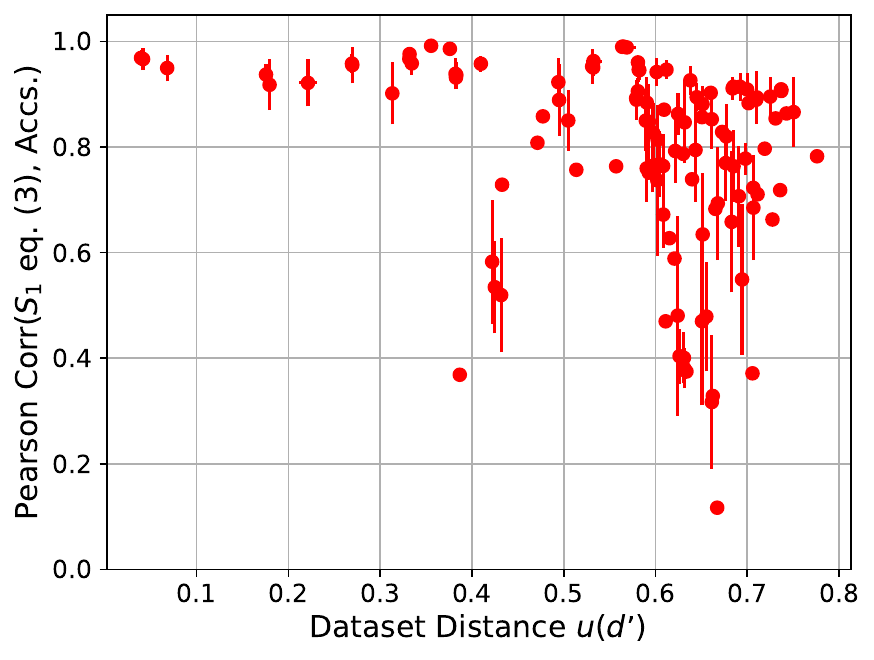}}
\subfigure[$S_2$ \eqref{eq:score_acc}]{\includegraphics[width=0.32\textwidth]{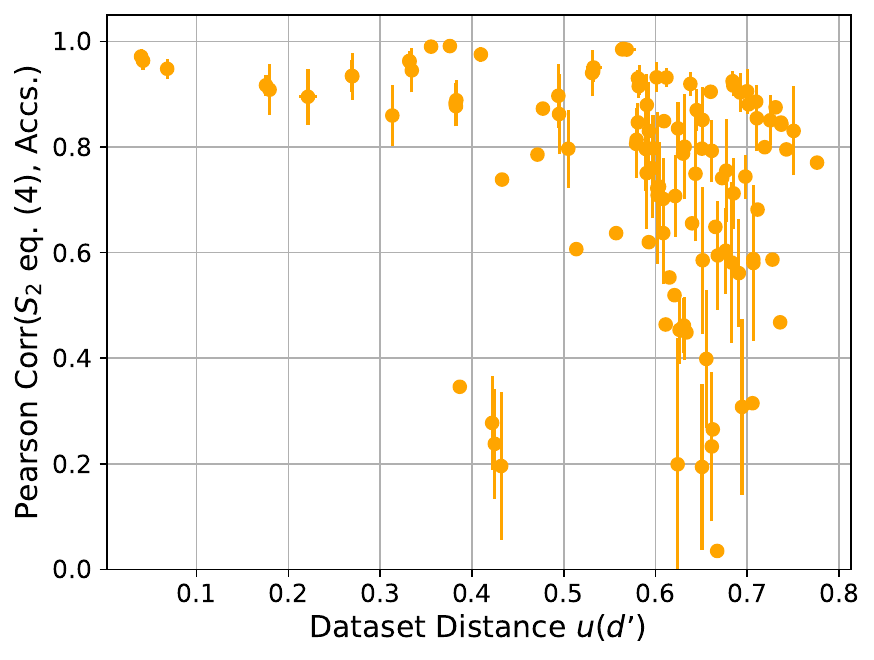}}
\subfigure[$S_3$ \eqref{eq:score_conf}; true $p(d',m)$]{\includegraphics[width=0.32\textwidth]{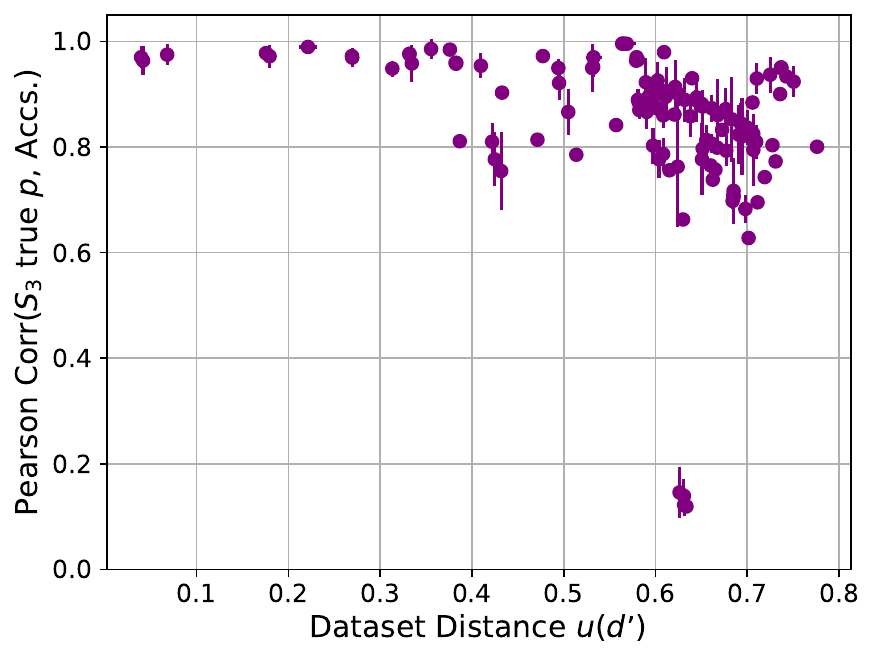}}
\vskip -0.1in
\caption{Correlation of scores and LLM accuracies on new tasks and corresponding data distances.}
\label{fig:distance-corr-sup}
\end{center}
\vskip -0.2in
\end{figure*}

\subsection{Dataset distance and Pearson correlation}
\label{sup:corr-distance}

The dataset distance $u(d')$ is computed as in \eqref{eq:data-distance}. As evident from \eqref{eq:point-distance}, dataset distance will usually decrease for larger values of $\alpha$ as inputs from $d'$ are moved into $\mathcal{D}$ (assuming that inputs from $d'$ are on average closer to each other than they are to inputs from other tasks). In this experiment, this serves as a mechanism to study the performance of LLM routing on closer datasets, providing insights into the benefits of learning LLM routers on \emph{more benchmarks} where it is more likely that dataset distance for a new task is small.

In Figure \ref{fig:distance-corr-sup} we present relations between dataset distance $u(d')$ and Pearson correlation between various model scores and accuracies of candidate LLMs. For results with $S_3$ see Figure \ref{fig:distance-corr}.

\section{Additional details for model routing on MixInstruct}
\label{sup:mix-instruct}
\paragraph{Correctness predictor and metrics} 
In the experiments on the MixInstruct dataset \citep{jiang2023llm}, we construct $S_1(m,d')$ following the scoring approach described in Appendix \ref{sup:data-distance}, where the MixInstruct train set was defined as the benchmark dataset. 
Then, instead of computing a per-dataset score for the entire test set, we compute the score for each test point, i.e., $S_1(m,x_i^{d'}) = g_m(x_i^{d'})$, and select a model $m$ per-point based on this score.
The reported metrics in Table \ref{tab:mixinstruct_per_point_results} and Figure \ref{fig:mixinstruct_per_dist_thresh} are averaged over the output evaluations of these per-point model selections. 
In our experiments, to compute $g_m(x_i^{d'})$ we use the BERTScore metric on the closest train set points (as $y(x,m)$ in \ref{eq:knn_classifier}). 
This was motivated by the conceptual relation between the implementation of our approach and the BERTScore, which relies on embedding space distances, and was validated empirically.

\paragraph{kNN parameter}
We set $k=10$ for the kNN classifier, slightly higher than in the HELM experiments. 
This choice was motivated by the in-distribution properties of the test set in MixInstruct, which is constructed from different parts of the same datasets that comprise the train set.
We note that the metrics did not significantly vary for different choices of $k\in [5,100]$.

\section{Proof of Lemma \ref{lem:adaptive-improvement}}
\label{sup:lemma}

\begin{lemma}[Lemma \ref{lem:adaptive-improvement}]
Let $\ell(y_1,y_2) = \rho(y_1 - y_2)$ for some subadditive $\rho:\mathbf{R}\to\mathbf{R}$ (e.g.\ $\rho(x) = \frac12x^2$ for the square loss). We have 
\[
\begin{aligned}
\ell(S_2,\widetilde{S}) \le \mathbf{E}\big[\ell(\bar{g}_m(X^d),y(X^d,m))\big]), \\
\ell(S_3,\widetilde{S}) \le \mathbf{E}\big[\ell(p(d,m)\bar{g}_m(X^d) + (1-p(d,m))(1-\bar{g}_m(X^d)),y(X^d,m))\big]).
\end{aligned}
\]
\end{lemma}

\begin{proof}
We start by showing the upper bound of $\ell(S_2,\widetilde{S})$:
\[
\begin{aligned}
\ell(S_2,\widetilde{S}) &= \rho(\mathbf{E}\big[\bar{g}_m(X^d)\big] - \mathbf{E}\big[y(X^d,m)\big]) & \text{(def of $\ell$)} \\
&= \rho(\mathbf{E}\big[\bar{g}_m(X^d) - y(X^d,m)\big]) \\
&\le \mathbf{E}\big[\rho(\bar{g}_m(X^d) - y(X^d,m))\big]) & \text{(convexity of $\rho$)}\\
&= \mathbf{E}\big[\ell(\bar{g}_m(X^d),y(X^d,m))\big]), & \text{(def of $\ell$)}
\end{aligned}
\]
where we recalled that subadditive functions are convex in the third step. The upper bound of $\ell(S_3,\widetilde{S})$ follows a similar argument:
\[
\begin{aligned}
\ell(S_3,\widetilde{S}) &= \rho(p(d,m)\mathbf{E}\big[\bar{g}_m(X^d)\big] + (1-p(d,m))(1-\mathbf{E}\big[\bar{g}_m(X^d)\big]) - \mathbf{E}\big[y(X^d,m)\big]) \\
&= \rho(\mathbf{E}\big[p(d,m)\bar{g}_m(X^d) + (1-p(d,m))(1-\bar{g}_m(X^d)) - y(X^d,m)\big]) \\
&\le \mathbf{E}\big[\rho(p(d,m)\bar{g}_m(X^d) + (1-p(d,m))(1-\bar{g}_m(X^d)) - y(X^d,m))\big]) \\
&= \mathbf{E}\big[\ell(p(d,m)\bar{g}_m(X^d) + (1-p(d,m))(1-\bar{g}_m(X^d)),y(X^d,m))\big]) \\
&\le \mathbf{E}\big[\ell(\bar{g}_m(X^d),y(X^d,m))\big]).
\end{aligned}
\]
\end{proof}

\newpage
\begin{table}[t]
\caption{HELM dataset details}
\label{tab:helm_dataset_desc}
\centering
{\renewcommand{\arraystretch}{1.1}
\begin{tabular}{| c || c | c |} 
 \hline
 Dataset & Size (instances) & Type \\ [0.1ex] 
 \hline\hline
 RAFT-ADE Corpus V2 & 40 & Binary Classification \\ 
 \hline
 RAFT-Banking 77 & 40 & 77 Class Classification \\
 \hline
 RAFT-NeurIPS Impact Statement Risks & 40 & Binary Classification \\
 \hline
 RAFT-One Stop English & 40 & 3 Class Classification \\
 \hline
 RAFT-Overruling & 40 & Binary Classification \\
 \hline
 RAFT-Semiconductor Org Types & 40 & 3 Class Classification \\
 \hline
 RAFT-Systematic Review Inclusion & 40 & Binary Classification \\
 \hline
 RAFT-TAI Safety Research & 40 & Binary Classification \\
 \hline
 RAFT-Terms of Service & 40 & Binary Classification \\
 \hline
 RAFT-Tweet Eval Hate & 40 & Binary Classification \\
 \hline
 RAFT-Twitter Complaints & 40 & Binary Classification \\
 \hline
 IMDB & 1000 & Binary Classification \\
 \hline
 Civil Comments-demographic=all & 1000 & Binary Classification \\
 \hline
 bAbI-QA-task=all & 1000 & Q\&A: one word answers \\
 \hline
 BoolQ & 1000 & Binary Classification \\
 \hline
 Entity Matching-Dataset=Beer & 182 & Binary Classification \\
 \hline
 Entity Matching-Dataset=Dirty iTunes Amazon & 218 & Binary Classification \\
 \hline
 Entity Matching-Dataset=Abt Buy & 1000 & Binary Classification \\
 \hline
 Entity Data Imputation-Dataset=Restaurant & 242 & Q\&A: one word answers \\
 \hline
 Entity Data Imputation-Dataset=Buy & 182 & Q\&A: one word answers \\
 \hline
 BBQ-subject=all & 1000 & Multiple Choice Questions \\
 \hline
 Legal Support & 1000 & Multiple Choice Questions \\
 \hline
 LSAT QA-task=all & 461 & Multiple Choice Questions \\
 \hline
 MMLU-Subject=Abstract Algebra & 111 & Multiple Choice Questions \\
 \hline
 MMLU-Subject=College Chemistry & 108 & Multiple Choice Questions \\
 \hline
 MMLU-Subject=Computer Security & 111 & Multiple Choice Questions \\
 \hline
 MMLU-Subject=Econometrics & 126 & Multiple Choice Questions \\
 \hline
 MMLU-Subject=US foreign policy & 111 & Multiple Choice Questions \\
 \hline
 Truthful QA-task=mc single & 654 & Multiple Choice Questions \\
 \hline\hline
 Total: $29$ datasets & 9946 & \\
 \hline
\end{tabular}
}
\end{table}

\begin{table}[t]
\caption{Candidate LLMs}
\label{tab:models}
\centering
{\renewcommand{\arraystretch}{1.1}
\begin{tabular}{| c || c | c |} 
 \hline
Name & Model Size, B & Average Accuracy on the 29 HELM tasks \\ [0.1ex]
\hline\hline
codegen-16b-mono & 16 & 0.451 \\
\hline
dial-flant5-xl & 3 & 0.454 \\
\hline
falcon-40b & 40 & 0.641 \\
\hline
flan-t5-xl & 3 & 0.650 \\
\hline
flan-t5-xxl & 11 & 0.658 \\
\hline
flan-ul2 & 20 & 0.668 \\
\hline
gpt-jt-6b-v1 & 6 & 0.576 \\
\hline
gpt-neox-20b & 20 & 0.492 \\
\hline
mpt-7b-instruct & 7 & 0.514 \\
\hline
mt0-xxl & 13 & 0.543 \\
\hline
llama-2-13b & 13 & 0.624 \\
\hline
llama-2-13b-chat & 13 & 0.623 \\
\hline
llama-2-13b-chat-beam & 13 & 0.603 \\
\hline
llama-2-70b & 70 & \textbf{0.688} \\
\hline
llama-2-70b-chat & 70 & \underline{0.687} \\
\hline
llama-2-7b & 7 & 0.610 \\
\hline
llama-2-7b-chat & 7 & 0.605 \\
\hline
starcoder & 15 & 0.587 \\
 \hline\hline
 Total: $18$ LLMs & 347 & \\
 \hline
\end{tabular}
}
\end{table}

\end{document}